\documentclass{dynafront2025} 

\usepackage{xspace}
\usepackage{amsmath}
\usepackage{amssymb}
\usepackage{multirow}
\usepackage{graphicx}  
\usepackage[table,xcdraw]{xcolor}
\usepackage{booktabs}
\usepackage{adjustbox} 
\usepackage{arydshln} 
\usepackage{wrapfig}
\newcommand{\ours}{\textsc{CSGD}\xspace}

\title[]{Composable Score-based Graph Diffusion Model for Multi-Conditional Molecular Generation}

\optauthor{%
\Name{Anjie Qiao} \Email{qiaoanj@mail2.sysu.edu.cn}\\
\Name{Zhen Wang$^*$} \Email{wangzh665@mail.sysu.edu.cn}\\
\Name{Chuan Chen} \Email{chenchuan@mail.sysu.edu.cn}\\
\addr $\text{1}.$ Sun Yat-sen University\\
\AND
\Name{DeFu Lian} \Email{liandefu@ustc.edu.cn}\\
\Name{Enhong Chen} \Email{cheneh@mail.ustc.edu.cn}\\
\addr $\text{2}.$ University of Science and Technology of China\\
}

\begin{document}

\maketitle

\begin{abstract}
Controllable molecular graph generation is essential for material and drug discovery, where generated molecules must satisfy diverse property constraints. While recent advances in graph diffusion models have improved generation quality, their effectiveness in multi-conditional settings remains limited due to reliance on joint conditioning or continuous relaxations that compromise fidelity. To address these limitations, we propose \textbf{C}omposable \textbf{S}core-based \textbf{G}raph \textbf{D}iffusion model (\ours), the first model that extends score matching to discrete graphs via concrete scores, enabling flexible and principled manipulation of conditional guidance. Building on this foundation, we introduce two score-based techniques: Composable Guidance (CoG), which allows fine-grained control over arbitrary subsets of conditions during sampling, and Probability Calibration (PC), which adjusts estimated transition probabilities to mitigate train–test mismatches. Empirical results on four molecular datasets show that \ours achieves state-of-the-art performance, with a 15.3\% average improvement in controllability over prior methods, while maintaining high validity and distributional fidelity. Our findings highlight the practical advantages of score-based modeling for discrete graph generation and its capacity for flexible, multi-property molecular design.
\end{abstract}


\section{Introduction}
\label{sec:intro}
Molecular graph generation is pivotal for material and drug discovery, enabling efficient exploration of vast chemical spaces~\cite{wang2025_into1,du2024machine_intro2,TNNLS_2/3Dgraph}. For real-world applications, a key challenge lies in controlling the generation process so that molecules satisfy multiple desired properties~\cite{gebauer2022_moleculesdesign,TNNLS_GraphCas}. For instance, one may wish to generate a set of candidate molecules that simultaneously meet specific ranges of synthetic scores and permeability to target gases.

Recent advances in graph diffusion models~\cite{vignac2022_digress,xu2024_disco,liu2024_graphdit} have significantly improved molecular generation quality. However, their ability in multi-conditional generation remains to be improved. Early approaches~\cite{bilodeau2022_earlywork,lee2023_MOOD} collapse all properties into a single joint condition (Figure~\ref{motivation}A), which obscures the influence of individual properties and hinders fine-grained control.
GraphDiT~\cite{liu2024_graphdit} mitigates this issue by designing separate encoders for categorical and numerical conditions (Figure~\ref{motivation}C), yet still treats the combination of encodings as a monolithic condition. This rigid design prevents dynamic inclusion or exclusion of conditions.

\begin{figure}[t]
    \centering
    \label{motivation}
    \includegraphics[width=1\textwidth]{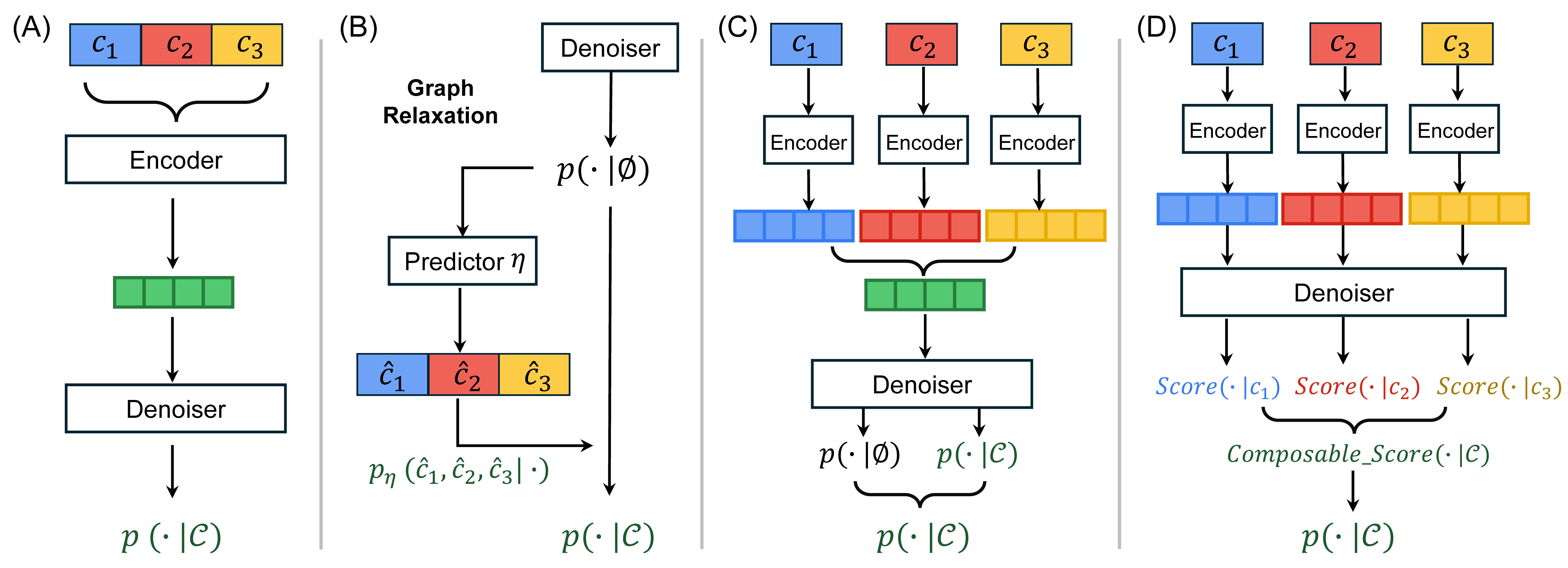}
    \caption{Comparison of multi-conditional guidance in discrete graph diffusion: (A) vanilla conditional generation; (B) predictor guidance in DiGress; (C) predictor-free guidance in GraphDiT; (D) our composable score-based guidance. For clarity, the sampling process and notation are simplified.}
\end{figure}

Therefore, a natural question arises: Is it viable to flexibly compose the guidance of each property considered for molecular graph generation?
In analogy to classifier guidance in continuous domain~\cite{dhariwal2021_classifiergui}, DiGress~\cite{vignac2022_digress} proposes predictor-guidance by combining the gradients from multiple property predictors (Figure~\ref{motivation}B).
However, this requires relaxing discrete graphs into continuous variables, potentially degrading generation fidelity.
Meanwhile, existing discrete graph diffusion models~\cite{vignac2022_digress,xu2024_disco,siraudin2024_cometh} typically adopt the mean prediction~\cite{austin2021structured_D3PM,campbell2022_absorb}, where the denoiser predicts the clean sample's distribution based on the current noisy 
and the given condition, rather than the score of the current marginal distribution.
This design makes score-based compositional guidance inapplicable~\cite{liu2022_composable}, and directly composing the estimated distributions of each considered condition lacks theoretical interpretation.

To overcome these limitations, we propose \textbf{C}omposable \textbf{S}core-based \textbf{G}raph \textbf{D}iffusion model (\ours), a novel model for flexible multi-conditional molecular generation. Built upon the recently proposed concrete score~\cite{lou2024_sedd}, we introduce concrete score-based diffusion model for discrete graphs, extending score matching to this domain. This design unlocks advanced score manipulation techniques. Specifically, we introduce \textit{Composable Guidance}, a novel mechanism that leverages the compositional nature of the score for multi-conditional generation (Figure~\ref{motivation}D). Then, we explore \textit{Probability Calibration}, manipulating conditional concrete scores to better shape the estimated transition probabilities within the generation process. With these techniques, \ours achieves flexible, property-aware sampling and significantly improves both validity and controllability.
In summary, our contributions are as follows:

\begin{itemize}
    \item To the best of our knowledge, we introduce the first concrete score-based graph diffusion model, which extends score matching to discrete graph generation and enabling advanced score-manipulation techniques that were previously inapplicable.
    \item We propose \textit{Composable Guidance} and \textit{Probability Calibration} for multi-conditional molecular generation, which manipulate per-property scores to steer generation toward satisfying desired properties, thereby improving the validity and controllability of the generated molecules.
    \item We compare \ours with various multi-conditional molecular generation methods (Figure~\ref{motivation}A-C) on four public datasets. \ours achieves state-of-the-art performance, with an average 15.3\% improvement in controllability, highlighting the effectiveness of score-based modeling and composable guidance.
\end{itemize}
\section{Related Works}
\label{sec:related}

\noindent\textbf{Graph Diffusion Models.}
Diffusion models~\cite{ho2020_DDPM,austin2021structured_D3PM,songscore_seds} have been widely applied to graph generation. Early approaches such as EDP-GNN~\cite{niu2020_EDPGNN} and GDSS~\cite{jo2022_GDSS} treated graphs as continuous variables, typically using a Gaussian reference distribution, similar to that in image generation. Because graphs are inherently discrete, subsequent work~\cite{vignac2022_digress,liu2024_graphdit,siraudin2024_cometh} modeled graphs in discrete state spaces to better capture their structural properties. Among them, DiGress~\cite{vignac2022_digress} demonstrated strong performance on various graph generation benchmarks and thus has been widely adopted in practice. However, these models generally operate in a discrete time setting, which limits the flexibility of sampling~\cite{xu2024_disco}. To address this, recent methods such as D\textsc{is}C\textsc{o}~\cite{xu2024_disco} and C\textsc{ometh}~\cite{siraudin2024_cometh} first introduce discrete-state, continuous-time diffusion models for graph generation. In this work, our proposed \ours also operates with continuous time and discrete state, while distinct from previous methods by introducing the concrete score~\cite{lou2024_sedd}, thereby unlocking score-manipulation techniques for guided sampling that would otherwise be unavailable for the existing mean prediction-based graph diffusion models.

\noindent\textbf{Guidance for diffusion model.}
Recent studies~\cite{hu2023_selfguidance,ahn2024_PAG,liu2022_composable} have shown that sampling guidance techniques, such as classifier guidance~\cite{dhariwal2021_classifiergui} and classifier-free guidance~\cite{ho2021_cfg}, are crucial for steering diffusion models toward samples that better respect satisfy condition(s). As classifier-free guidance has shown its superiority in many tasks, many recent works proposed improvements to it, including applying guidance only during the middle interval of sampling to improve inference speed and quality~\cite{kynkaanniemi2024_guidanceinterval}, guiding a high-quality model with a degraded, lower-quality model trained on the same task and data~\cite{karras2024_guidebyitself}, a first-principles nonlinear correction~\cite{zheng2024_characteristicguidance}, sliding window guidance~\cite{kaiser2024_slideguidance}, and independent condition guidance~\cite{sadatno_ICG}. These techniques have proven effective in improving the sample quality and condition obedience. Building on our introduction of concrete score, we further propose two score-manipulation techniques to improve the validity and controllability of discrete graph generation.
\section{Methodology}
\label{sec:method}
In this section, we present \ours, which tailors SEDD~\cite{lou2024_sedd} to the domain of multi-conditional molecular generation (see Section~\ref{subsec:score}). Leveraging the composable nature of the scores, \ours enables more flexible and controllable sampling strategies for generation tasks conditioning on multiple properties (see Section~\ref{subsec:composable}). In analogy to recently proposed tricks for  scores in the continuous domain, we propose Probability Calibration (PC) mechanism for \ours to improve the validity and controllability of the generated molecular graphs (see Section~\ref{subsec:calibration}).

\noindent\textbf{Notations and Problem Definition.}
We model graphs in a discrete space, where nodes and edges take categorical attributes, denoted by $\mathcal{X}=\{1,\ldots,a\}$ and $\mathcal{E}=\{1,\ldots,b\}$, respectively. A graph is represented as $G=(\mathbf{X},\mathbf{E})$, where $\mathbf{X}=(x^i)_{i=1}^{n} \in\mathcal{X}^n$ specifies each node's type (i.e., atom type) and $\mathbf{E} =(e^{ij})_{i,j=1}^{n}\in \mathcal{E}^{n \times n}$ specifies each edge's type (i.e., bond type or the indicator of no connection). Here, $n$ denotes the number of nodes.
Let $\mathcal{C}=\{c_1,c_2,\dots,c_M\}$ denote a set of $M$ specified conditions, where $c_i$ may be numerical or categorical.
The multi-conditional molecular generation is to model: $p(G|\mathcal{C})=p(G| c_1,\dots,c_M ) \propto p(G)p(c_1,\dots,c_M | G)$, 
where molecular graph $G$ can be evaluated along two dimensions: 
(1) Distribution Learning $p(G)$: measuring the fidelity to the underlying unconditional distribution; 
(2) Controllability $p(c_1,\dots,c_M | G)$: measuring consistency between generated molecular graph and the specified conditions.
Moreover, there are often practical scenarios in which we need to sample $G$ with only a subset of properties specified.
A flexible solution is demanded that can handle an arbitrary subset of properties based on one common model.

\subsection{Score-based Graph Diffusion Model}
\label{subsec:score}
For brevity, we present our diffusion model in an unconditional setting to focus on the use of the concrete score~\cite{meng2022_concretescore}. This unconditional formulation can be readily adapted to our target multi-conditional generation problem simply by feeding the encoded condition(s) into the denoiser.

\noindent\textbf{Diffusion Process.}
Following prior graph diffusion models~\cite{xu2024_disco}, we diffuse each node and each edge independently following a continuous-time Markov process, which can be described by linear ODEs~\cite{lou2024_sedd}:
\begin{align}
\frac{\mathrm{d}p_{t}^{X}}{\mathrm{d}t}= Q_t^Xp_{t}^{X},\quad\frac{\mathrm{d}p_{t}^{E}}{\mathrm{d}t}= Q_t^Ep_{t}^{E},
\label{eq:fwdpfode}
\end{align}
where $\forall t\in\Re_{+},p_{t}^{X}\in\Re^{a},p_{t}^{E}\in\Re^{b}$ denote the marginal discrete probability distributions, with boundary conditions $p_0 \approx p_{\text{data}}$ and $p_{t\to\infty}\approx p_{\text{base}}$.
With a scalar noise schedule $\sigma(t)$~\cite{ho2020_DDPM}, the time-dependent transition matrices are defined as $Q_t^X = \sigma(t) Q^X$ and $Q_t^E = \sigma(t) Q^E$, where $Q^X\in\mathbb{R}^{a\times a}$ and $Q^E\in\mathbb{R}^{b\times b}$ are the transition matrices for nodes and edges, respectively.
Due to this independence, the likelihood of a graph $G_t$ at timestep $t$ is factorized as $p_t(G_t)=\prod_{i=1}^{n}p_{t}^{X}(x_{t}^i)\prod_{i,j=1}^{n}p_{t}^{E}(e_{t}^{ij})$.
We adopt two common choices for the transition matrix $Q^X$ (resp. $Q^E$): \textit{Absorb}~\cite{campbell2022_absorb}, which introduces a dedicated absorbing (MASK) state, and \textit{Uniform}~\cite{austin2021structured_D3PM}, which assigns the same transition probabilities to all other states.

Then, the element-wise forward transition admits a closed form. Conditioned on the state at $t=0$, the probability is given by the cumulative transition matrix:
\begin{equation}
\begin{aligned}
p_{t|0}^{X}(\cdot | x^{i}_0 = x) &= x\text{-th column of}\exp{(\bar{\sigma}(t)Q^X)},\\ 
p_{t|0}^{E}(\cdot | e^{ij}_0 = e) &= e\text{-th column of}\exp{(\bar{\sigma}(t)Q^E)},
\end{aligned}
\label{eq:singlefwdt0}
\end{equation}
where $\bar{\sigma}(t) = \int_0^t \sigma(s) \,\mathrm{d}s$ is the cumulative noise coefficient.  
Thus, the likelihood of the entire graph $G_t$ conditioned on the observation $G_0$ is given by $
p_{t|0}(G_t | G_0)
= \prod_{i=1}^{n} p_{t|0}^{X}(x_{t}^{i} | x_{0}^i)
  \;\prod_{i,j=1}^{n} p_{t|0}^{E}(e_{t}^{ij} | e_{0}^{ij})$.

\noindent\textbf{Reverse Process.}
The reverse process aims to reconstruct graphs obeying $p_{\text{data}}$ from noisy ones sampled from $p_{\text{base}}$, which can be completed via simulating the reversal of Eq.~\eqref{eq:fwdpfode}, that is, $\frac{\mathrm{d}p_{T-t}^{X}}{\mathrm{d}t}= \bar{Q}_{T-t}^X p_{T-t}^{X}$ and $\frac{\mathrm{d}p_{T-t}^{E}}{\mathrm{d}t}= \bar{Q}_{T-t}^Ep_{T-t}^{E}$~\cite{kelly1981_Kelly,sunscore}.
Since the flux should be consistent, that is, $\bar{Q}_{t}^{X}(x',x)p_{t}^{X}(x) = Q_{t}^{X}(x,x')p_{t}^{X}(x')$, existing works~\cite{meng2022_concretescore} have proposed to estimate the ratio $\frac{p_{t}^{X}(x')}{p_{t}^{X}(x)}$, namely concrete score, and use these scores to perform the reverse process. Edges can be treated analogously via \(\frac{p_{t}^{E}(e')}{p_{t}^{E}(e)}\).

Due to the combinatorial nature of discrete graphs, a general transition matrix defining probabilities between all possible graph pairs is clearly intractable. Thanks to the independence design of our diffusion process, we restrict each transition to flip only a single position—either a node or an edge—at a time.
Under this design, the concrete scores can be naturally parameterized using a neural network $s_{\theta}(\cdot,\cdot)$:
\begin{equation}
\begin{aligned}
s_\theta^{X}(G_t, t)_{i,\tilde{x}_t^i} &\approx
 \frac{p_{t}^{X}(x_t^{1} \ldots \tilde{x}_t^i \ldots x_t^{n})}{p_{t}^{X}(x_t^1 \ldots x_t^{i} \ldots x_t^n)}\text{ for }1\leq i \leq n,\tilde{x}_{t}^{i}\neq x_{t}^{i},\\
s_\theta^{E}(G_t, t)_{ij,\tilde{e}_t^{ij}} &\approx \frac{p_{t}^{E}(e_t^{11} \ldots  \tilde{e}_t^{ij} \ldots e_t^{nn})}{p_{t}^{E}(e_t^{11} \ldots e_t^{ij} \ldots e_t^{nn})}\text{ for }1\leq i,j \leq n,\tilde{e}_{t}^{ij}\neq e_{t}^{ij}.
\end{aligned}
\end{equation}
Here, $x_{t}^i$ (resp. $e_{t}^{ij}$) denotes the current node (resp. edge) state and $\tilde{x}_{t}^i$ (resp. $\tilde{e}_{t}^{ij}$) are possible alternative states from the corresponding discrete space.
Specifically, $s_{\theta}(\cdot,\cdot)$ is often a Graph Transformer that processes both node tokens and edge tokens, where the subscripts $(i,\tilde{x}_{t}^{i})$ mean to take its $\tilde{x}_{t}^{i}$-th output dimension at the $i$-th node token, with edges treated analogously.

Following SEDD, we borrow the derivation of Tweedie denoising~\cite{efron2011_tweedie} yet perform $\tau$-leaping, with each token states $x_{s}^{i}$ and $e_{s}^{ij}$ at timestep $s=t-\Delta t, \Delta t \geq 0$ having their conditional probabilities as follows:
\begin{equation}
\begin{aligned}
p_{s|t}^{X}(x^i_{s}|x^i_{t} ; \theta) &= \left(\operatorname{exp}(-\sigma_t^{\Delta t}Q^X)s_{\theta}^X(G_t,t)_{i}\right)_{x^i_{s}}\operatorname{exp}(\sigma_t^{\Delta t}Q^X)(x^i_{t},x^i_{s}), 
\\
p_{s|t}^{E}(e^{ij}_{s}|e^{ij}_{t};\theta) &= \left(\operatorname{exp}(-\sigma_t^{\Delta t} Q^E)s_{\theta}^E(G_{t},t)_{ij}\right)_{e^{ij}_{s}}\operatorname{exp}(\sigma_t^{\Delta t})Q^E)(e^{ij}_{t},e^{ij}_{s}),
\label{eq:reversetransition}
\end{aligned}
\end{equation}
where $\sigma_t^{\Delta t} = \overline{\sigma}(t) - \overline{\sigma}(s)$.
Specifically, we perform Euler step for each token independently based on Eq.~\eqref{eq:reversetransition} yet with $\tau$-leaping. Essentially, the reverse step factorizes across tokens:
\begin{align}
\label{eq:prob}
p_{s|t}(G_{s}|G_t; \theta)=\prod_{ i=1}^n p_{s|t}^{X}(x^i_{s}|x^i_t ; \theta) \prod_{i,j =1}^n p_{s|t}(e^{ij}_{s}|e^{ij}_t ; \theta).
\end{align}

\noindent\textbf{Optimization.}
The denoiser $s_{\theta}$ is trained to predict concrete scores by minimizing the \textit{Diffusion Weighted Denoising Score Entropy}~\cite{lou2024_sedd}:
\begin{equation}
\begin{aligned}
\label{lossfunction}
\mathcal{L}_{t}^{X}(G_0,G_t,\theta) = \sigma(t) \sum_{i=1}^n \sum_{ \tilde{x}^i_t=1}^a (1 - \delta_{x^i_t}(\tilde{x_t}^i)) \left(s_{\theta}^X(G_t, t)_{i,\tilde{x}^i} - \frac{p_{t|0}(\tilde{x_t}^i|x^i_0)}{p_{t|0}(x^i_t|x^i_0)}\right) \log s_{\theta}^X(G_t, t)_{i,\tilde{x}^i},
\\
\mathcal{L}_{t}^{E}(G_0, G_t, \theta) = \sigma(t) \sum_{i,j=1}^n\sum_{ \tilde{e}_t^{ij}=1}^b (1 - \delta_{e^{ij}_t}(\tilde{e_t}^{ij})) \left(s_{\theta}^E(G_t, t)_{ij,\tilde{e}^{ij}} - \frac{p_{t|0}(\tilde{e_t}^{ij}|e^{ij}_0)}{p_{t|0}(e^{ij}_t|e^{ij}_0)}\right) \log s_{\theta}^E(G_t, t)_{ij,\tilde{e}^{ij}}.
\end{aligned}
\end{equation}
The Kronecker delta $\delta_{x_{t}^{i}}(\tilde{x}^i_t)$ and $\delta_{e_{t}^{ij}}(\tilde{e}^{ij}_t)$ ensures that the current state is excluded from the summation. The final loss combines node and edge terms with weight $\lambda$ as $
\mathcal{L}_{t}(G_0, G_t, \theta) = \mathcal{L}_{t}^{X}(G_0, G_t, \theta) + \lambda \mathcal{L}_{t}^{E}(G_0, G_t, \theta)$.
The score entropy loss is well-suited for recovering the ground-truth concrete scores~\cite{lou2024_sedd}.
We optimize it using the Adam~\cite{2015kingma_adam} optimizer, with gradient clipping and a learning rate warmup~\cite{he2016deep_warmup} to stabilize training.

\noindent\textbf{Analysis.} Suppose we discretize the reversal of Eq.~\ref{eq:fwdpfode} by timesteps $0 = t_0 < t_1 < \cdots < t_{T-1} < t_T = 1$ and simulate the process to sample $G_0$, the corresponding estimated distribution is $p(G;\theta)=\sum_{G_{t_1},\ldots,G_{t_T}}\prod_{k=1}^{T}p_{t_{k-1}|t_k}(G_{t_{k-1}}|G_{t_k};\theta)p_{\text{base}}(G_{t_T})$.
Since isomorphic molecular graphs are treated equally, the permutation-invariant $p(G;\theta)$ is preferred.

\begin{proposition}
The estimated distribution $p(G;\theta)$ is permutation-invariant.
\end{proposition}
\begin{proof}
According to the theoretical results of GeoDiff~\cite{xu2022geodiff}, we just need to show that (1) the base distribution $p_{\text{base}}$ is permutation-invariant and (2) the transitions $p_{t_{k-1}|t_k}(G_{t_{k-1}} | G_{t_k};\theta)$ are permutation-equivariant.

At first, no matter whether to choose Absorb or Uniform, $p_{\text{base}}$ is factorized over i.i.d. node and edge distributions.

Our denoiser is implemented as a Graph Transformer~\cite{liu2024_graphdit}, which is permutation-equivariant: any permutation $\pi_{n}$ of the input tokens induces the corresponding permutation $\pi_n$ on its output tokens.
Hence, the transitions $p_{t_{k-1}|t_k}(G_{t_{k-1}} | G_{t_k};\theta)$ is permutation-equivariant according to Eq.~\ref{eq:reversetransition}.

As these two conditions are satisfied, we have that $p(\pi_{n}.G;\theta) = p(G;\theta)$.
\end{proof}

\subsection{Multi-conditional Molecular Generation via Concrete Scores}
\label{subsec:composable}
As discussed at the beginning of this Section, we focus on multi-conditional molecular generation, where the observed dataset $\mathcal{D}=\{(G,\mathcal{C})_l\}_{l=1}^{N}$ is an unbiased sample drawn from the underlying joint distribution.
Our goal is to estimate the conditional distribution $p(G|\mathcal{C})$ so that we can sample molecular graphs that are valid and respect the specified properties $\mathcal{C}$.
We will elaborate on how to achieve this goal with our score-based graph diffusion model and further to enable the flexible specification of an arbitrary subset of the $M$ properties.

\noindent\textbf{Condition Encoders.} 
To handle multi-conditional generation tasks that involve both categorical and numerical properties, we adopt separate encoders for each property, following GraphDiT. Specifically, each categorical property is processed by a \textit{Categorical Encoder}, and each numerical property by a \textit{Cluster Encoder}. See Appendix~\ref{appendix:encoder} for details of these encoders. The resulting embeddings can then be combined in different ways, such as aggregating all property embeddings to form a joint representation with the mean pooling (see Figure~\ref{motivation}C) or using a single property embedding as the input to the denoiser (see Figure~\ref{motivation}D).

\noindent\textbf{Classifier-free guidance (CFG).}
We first employ the same classifier-free guidance (CFG) as GraphDiT (see Figure~\ref{motivation}C).
During training and inference, all $M$ properties are encoded and combined into a joint condition embedding, which is fed into the backbone of the denoiser to produce multi-conditional score. Then, both this conditional score and the unconditional one are merged to guide generation. Specifically, given a guidance scale $w$, the CFG score is:
\begin{align}
s_\theta^{\mathrm{CFG}}(G_t, \mathcal{C}, t)
= s_\theta(G_t,\varnothing,t)
+ w \,\bigl(s_\theta(G_t,\mathcal{C} ,t) - s_\theta(G_t,\varnothing, t)\bigr).
\end{align}
As can be seen, CFG treats all properties as a whole $\mathcal{C}$, producing a single monolithic score $s_\theta(G_t, \mathcal{C}, t)$ and limiting fine-grained control over individual properties $c_m$.

\begin{figure}[t]
\centering
\includegraphics[width=1\textwidth]{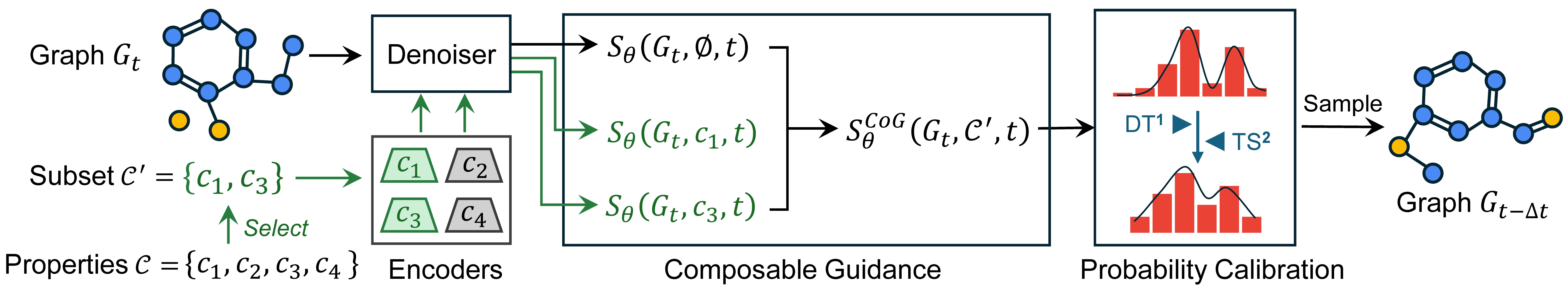}
\caption{An overview of \ours with Composable Guidance (CoG) to enable flexible control ($\text{P}.$: Gas Permeability; $\text{Synth.}$: Synthesizability). $^1$ Dynamic Thresholding, $^2$ Temperature Scaling.
}
\label{fig:framework}
\end{figure}

\noindent\textbf{Composable Guidance (CoG).}
We illustrate one simulation step of \ours with CoG in Figure~\ref{fig:framework}, where the goal is to flexibly generate molecules conditioned on arbitrary subsets of all the $M$ properties. To this end, we rely on the decomposition of multi-conditional probability:
\begin{align}
\label{cog1}
p(G | c_1,\ldots,c_M) \propto p(G)p(c_1,\ldots,c_M | G) = p(G)\prod_{m=1}^{M}p(c_m | G) = p(G)\prod_{m=1}^{M}\frac{p(G | c_m)p(c_m)}{p(G)},
\end{align}
where we assume conditional independence of the properties given $G$, and we substitute each conditional likelihood $p(c_m | G)$ by: $p(c_m | G) = \frac{p(G | c_m)p(c_m)}{p(G)}$.
Following prior work on compositional generation~\cite{liu2022_composable}, we adopt this conditional independence assumption to make score composition tractable, while recognizing that handling correlated conditions remains an important direction for future research.

This decomposition leads to the definition of our CoG score~\cite{liu2022_composable}:
\begin{align}
\label{cog}
s_\theta^{\mathrm{CoG}}(G_t,\mathcal{C},t)
&= s_\theta(G_t,\varnothing,t)
  \;+\;\sum_{m=1}^{M} w_m \Bigl(s_\theta(G_t,c_m,t) - s_\theta(G_t,\varnothing,t)\Bigr).
\end{align}
This CoG strategy modulates the contribution of each single property-specific score $s_\theta(G_t,c_m,t)$ via a corresponding guidance scale $w_m$.
Such a formulation provides fine-grained control, enabling selective emphasis on specific subsets of properties during sampling.

To enable CoG, we train a the denoiser $s_{\theta}$ following the procedure described in Algorithm~\ref{algo.train}. During training, each condition $c_m$ is picked with a fixed probability, $s_{\theta}$ is enabled to jointly learn the conditional score functions $s_{\theta}(G_t, c_m, t)$ for each $c_m \in \mathcal{C}$ and the unconditional one $s_{\theta}(G_t, \varnothing, t)$ in a unified way. At inference time (Algorithm~\ref{algo.sample}), these scores are combined via the CoG strategy to guide generation toward arbitrary subsets of target properties.

\noindent\textbf{Efficiency Concern} 
With CoG, a separate score must be computed for each property, that is to say, the inference time increase linearly with the number of specified properties, which would be costly when this number is large. To eliminate this issue, we propose a training strategy that enables both CoG and a more efficient variant, termed \textbf{fast-CoG}.

During training, instead of uniformly sampling one property $c_m$ (Line3 in Algorithm~\ref{algo.train}), we randomly sample an non-empty subset $\mathcal{C}_{\text{sub}}\subseteq\mathcal{C}$ of the $M$ specified properties and aggregate their embeddings via mean pooling to form a single combined condition, which will be fed into the denoiser. The model is trained to predict the corresponding multi-conditional score $s_\theta(G_t,\mathcal{C}_{\text{sub}} ,t)$ based on this aggregated embedding. At inference time, $s_{\theta}$, learned in this way, enables us to compute the desired multi-conditional score for any subset of specified properties with one time of model inference, which can then be directly used to make state transition. Formally, the fast-CoG score is defined as:
\begin{align} s_\theta^{\mathrm{fast-CoG}}(G_t, \mathcal{C}_{\text{sub}}, t) = s_\theta(G_t,\varnothing,t) + w \,\bigl(s_\theta(G_t,\mathcal{C}_{\text{sub}} ,t) - s_\theta(G_t,\varnothing, t)\bigr).
\end{align}
Unlike CFG, fast-CoG allows flexible selection of property subsets during sampling.
This training strategy enables the model to support both CoG and fast-CoG at inference, providing an option for scenarios that prioritize efficiency.

\begin{algorithm}[tb]
\par\noindent\rule{\linewidth}{1.0pt}\par
\caption{Training a Score-based Graph Diffusion Model for CoG.}
\par\noindent\rule{\linewidth}{0.5pt}\par
\label{algo.train}
\KwIn{%
  Denoiser $s_{\theta}$, dataset $\mathcal{D}=\{(G, \mathcal{C})_l\}_{l=1}^{N}$, 
  transition matrices $Q^{X},Q^{E}$, noise schedule $\bar{\sigma}$, loss weight $\lambda$.
  }
\KwOut{Trained denoiser $s_{\theta}$.}
\While{until convergence}{
    Get a pair $(G,\mathcal{C})$ from $\mathcal{D}$ and let $G_0 \gets G$;\;
    
    Sample timestep $t\sim \mathcal{U}(0,1)$ and property $m \sim \mathcal{U}(\{1,\ldots,M\})$;\;
    
    Construct noisy graph $G_t$: $x^{i}_t \sim p_{t|0}^{X}(\cdot | x^{i}_0)$, $e^{ij}_t \sim p_{t|0}^{E}(\cdot | e^{ij}_0)$ based on Eq.~\ref{eq:singlefwdt0};\;
 
    Randomly set $c_m\gets \varnothing$ with a probability to enable unconditional training;\;
    
    Predict concrete scores: $(s_{\theta}^X,s_{\theta}^E) \gets s_{\theta}(G_t,c_m,t)$;\;
    
    Compute loss $\mathcal{L}_{t}(G_0, G_t, \theta) = \mathcal{L}_{t}^X(G_0, G_t, \theta) + \lambda \mathcal{L}_{t}^E(G_0, G_t, \theta)$ based on Eq.~\ref{lossfunction};\;
    
    Update denoiser $s_\theta$ via gradients $\nabla_{\theta}\mathcal{L}_{t}$.\;
}
\par\noindent\rule{\linewidth}{1.0pt}\par
\end{algorithm}

\begin{algorithm}[tb]
\par\noindent\rule{\linewidth}{1.0pt}\par
\caption{Sampling from Score-based Graph Diffusion Model with CoG.}
\par\noindent\rule{\linewidth}{0.5pt}\par
\label{algo.sample}
\KwIn{Trained denoiser $s_\theta$, base distribution $p_{\text{base}}$, transition matrices $Q^X,Q^E$, noise schedule $\bar{\sigma}$, specified timesteps $0=t_0 < \cdots < t_{T}=1$, specified properties $\mathcal{C}'=\{c_{m_1},\ldots,c_{m_L}\},1\leq m_1 < \cdots < m_L \leq M$.}
\KwOut{Sampled graph $G$}

Sample $G_T \sim p_\text{base}$, $k\gets T$;\;

\While{$k>0$}{
    Compute unconditional score $s_{\theta}(G_t, \varnothing, t_k)$;\;
    
    \For{each $1\leq l \leq L$}{Compute conditional score: $s_{\theta}(G_t, c_{m_l}, t_k)$;\;}

    Compute CoG scores $s_{\theta}^{\text{CoG}}(G_t, \mathcal{C}', t_k)$ based on Eq.~\ref{cog};\;

    Construct transition densities $p_{t_{k-1}|t_{k}}^{X}(x_{t_{k-1}}^i|x_{t_k}^{i};\theta)$ and $p_{t_{k-1}|t_k}^{E}(e_{t_{k-1}}^{ij}|e_{t_k}^{ij};\theta)$ by substituting $s_{\theta}$ by $s_{\theta}^{\text{CoG}}$ based on Eq.~\ref{eq:reversetransition}; \;

    Apply PC to transition densities $p_{t_{k-1}|t_{k}}^{X}(x_{t_{k-1}}^i|x_{t_k}^{i};\theta)$ and $p_{t_{k-1}|t_k}^{E}(e_{t_{k-1}}^{ij}|e_{t_k}^{ij};\theta)$ as described in Section~\ref{subsec:calibration}\tcp*{Optional}
    
    \For{each $1\leq i \leq n$}{
        Sample $x_{t_{k-1}}^i \sim p_{t_{k-1}|t_{k}}^{X}(\cdot|x_{t_k}^{i};\theta)$;\;
    }
    \For{each $1\leq i, j \leq n$}{
        Sample $e_{t_{k-1}}^{ij} \sim p_{t_{k-1}|t_k}^{E}(\cdot|e_{t_k}^{ij} ; \theta)$;\;
    }

    Construct $G_{t_{k-1}} \gets ((x_{t_{k-1}}^i)_{i=1}^{n},(e_{t_{k-1}}^{ij})_{i,j=1}^{n})$;\;
    
    Update $k \gets k-1$;\;
}
\Return $G \gets G_0$\;
\par\noindent\rule{\linewidth}{1.0pt}\par
\end{algorithm}

\subsection{Probability Calibration with Concrete Scores}
\label{subsec:calibration}
After obtaining the guided score, we compute the transition probabilities $p^{\theta}_{s|t}(G_{s} | G_t; \theta)$ according to Eq.~\eqref{eq:prob}. Although guidance improves condition alignment, it may also compromise sample fidelity. In particular, high guidance scales can cause a train–test mismatch~\cite{saharia2022_imagen}, where the guided score, i.e., $s_\theta^{\mathrm{CoG}}$, during sampling falls outside the ranges observed in training. In our work, we observe a similar phenomenon which degrades both the validity and controllability of generated molecules. To address this, we apply \emph{dynamic thresholding} and \emph{temperature scaling}, which jointly calibrate the probabilities.

\noindent\textbf{Dynamic Thresholding.} Let $p_{\text{raw}}$ denote the raw guided probabilities computed with $s_{\theta}^{\text{CoG}}$ based on Eq.~\ref{eq:reversetransition}. We first clamp negative entries to zero and compute the lower and upper percentiles (e.g., $\alpha=1\%$, $\beta=99\%$) across the state dimension:
$
\ell = \mathrm{Quantile}(p_{\mathrm{raw}}, \alpha),
h = \max\bigl(\mathrm{Quantile}(p_{\mathrm{raw}}, \beta),\,\ell + \varepsilon\bigr),
$
where $\varepsilon$ is a small constant (e.g., $10^{-6}$) to avoid division by zero. The probabilities are then normalized by min–max scaling:
\begin{align}
p_{\mathrm{clipped}}
= \frac{\bigl[p_{\mathrm{raw}} - \ell\bigr]_{+}}{\,h - \ell\,},
\qquad
[p]_{+} \equiv \max(p, 0).
\end{align}
This procedure guarantees valid probabilities in $[0,1]$, stabilizing sampling by preventing extreme values from dominating.
Our dynamic thresholding serves as a discrete analogue to Imagen’s method~\cite{saharia2022_imagen}: rather than scaling continuous pixel values, we calibrate discrete transition probabilities by percentile-based clipping, ensuring well regularized probability ranges during sampling.

\noindent\textbf{Temperature Scaling.}  
Temperature scaling is a post-processing method used to improve probabilistic predictions in classification tasks~\cite{guo2017_tempscal1,kull2019beyond_tempscal2,xuan2025_tempscal3}. Its impact on graph diffusion sampling remains unexplored. Here, we investigate whether temperature scaling can improve the validity and controllability of generated graphs by modulating the sharpness of the predicted distributions at each step of sampling.
To this end, we further adjust the distribution via temperature $\tau$:
\begin{align}
p_{\mathrm{scaled}}(i)
= \frac{\bigl(p_{\mathrm{clipped}}(i)\bigr)^{1/\tau}}
       {\sum_{j} \bigl(p_{\mathrm{clipped}}(j)\bigr)^{1/\tau}}\text{ for each outcome $i$}.
\end{align} 
Lower $\tau$ (i.e., $0<\tau<1$) sharpens the distribution and favors high-probability states, while higher $\tau$ (i.e., $\tau>1$) flattens it and promotes diversity.

Probability Calibration combines percentile thresholding and temperature scaling to calibrate guided probability, enhancing both validity and controllability in molecular graph generation. 
In practice, it involves three key parameters: $\alpha$, $\beta$, and $\tau$. Based on our experience, $\tau$ is the parameter most worth tuning, while $\alpha$ and $\beta$ can be assigned reasonable values to ensure stability (e.g., $\alpha = 1\%$, $\beta = 99\%$).
\section{Experiments}
\label{sec:exp}
To evaluate the effectiveness of extending score matching to discrete graph diffusion models, as well as the score-based capacities it enables, namely, composable guidance (CoG) and probability calibration (PC), we conducted comprehensive quantitative experiments to answer the following research questions: \textbf{RQ1}: Does \ours outperform related discrete graph diffusion models? \textbf{RQ2}: Can concrete scores improve the quality and controllability of samples relative to the mean prediction approach? \textbf{RQ3}: How beneficial are CoG and PC for generating valid and controllable samples in multi-conditional generation? \textbf{RQ4}: Can CoG support flexible sampling under varying numbers and types of molecular property conditions?
Then, we assess the sampling efficiency of \ours with various guidance strategies to show its versatility and flexibility.

\subsection{Experimental Setup}
\noindent  \textbf{Datasets.} Following the experimental setup in GraphDiT~\cite{liu2024_graphdit}, we evaluate multi-conditional molecular generation using two types of datasets: (1) For materials design, we use a Polymers dataset~\cite{yampolskii2012_polymers}, which includes four numerical conditions related to synthesizability—Synthetic Accessibility Score (SAS)~\cite{ertl2009_sas} and Synthetic Complexity Score (SCS)~\cite{coley2018_scs}—as well as gas permeability properties ($\text{O}_2$P., $\text{N}_2$P. and $\text{CO}_2$P.). (2) For drug design, we evaluate on three class-balanced dataset--BACE, BBBP, and HIV--from MoleculeNet~\cite{wu2018_moleculenet}. Each dataset provides a categorical condition specific to its task (e.g., blood–brain barrier permeability for BBBP, $\beta$-secretase 1 inhibition for BACE, and HIV replication inhibition for HIV), together with the same numerical synthesizability conditions (SAS and SCS) used in the Polymers dataset.

\noindent  \textbf{Metrics.}    To systematically evaluate model performance, we use over three classes of metrics covering generation quality, distribution Learning, and controllability. First, we assess \textbf{Validity}, defined as the proportion of valid molecules among generated samples. Second, we evaluate the model's ability to learn the target distribution using three metrics: 
\textbf{Diversity} (measured by Tanimoto similarity~\cite{bajusz2015_tanimoto} among generated molecules), \textbf{Similarity} (fragment-based similarity to the reference set), and \textbf{FCD}~\cite{preuer2018_FCD} (Fréchet ChemNet Distance) between generated and reference distributions. Third, we measure controllability: for numerical conditions, we report the \textbf{MAE} between the specified conditions and predicted properties of the generated molecules; for categorical conditions, we report \textbf{Accuracy}. 

\noindent \textbf{Baselines.}   We compare our method with recent diffusion-based models for multi-conditional molecular generation, including GDSS~\cite{jo2022_GDSS}, DiGress~\cite{vignac2022_digress}, and their conditional variants MOOD~\cite{lee2023_MOOD} and DiGress v2, as well as GraphDiT~\cite{liu2024_graphdit}, which represents the current state-of-the-art in this task. For MOOD and DiGress v2, trained multi-task property predictors with the same architecture are used to provide guidance.
We note that GraphDiT adopts a graph-dependent forward diffusion process, which demonstrates superior performance within its full model pipeline~\cite{liu2024_graphdit}.
This creates a noise-model mismatch relative to our independent tokenwise scheme. If anything, this discrepancy may favor GraphDiT, making our evaluation of \ours conservative and thus fair under these conditions.

\noindent \textbf{Implementation Details.} 
Our implementation builds on the official open-source code of GraphDiT~\cite{liu2024_graphdit}, leveraging its Graph Transformer and condition encoder to handle both numerical and categorical properties. We split each dataset into training, validation, and test sets with a 6:2:2 ratio. For evaluation, we generate 10,000 molecules conditioned on the properties of the reference set.
Generation validity and distributional metrics are computed using the MOSES~\cite{polykovskiy2020_moses}. For controllability evaluation, we follow GraphDiT and adopt a random forest model trained on all task-related molecules as an Oracle~\cite{gao2022_oracle}. The properties of samples are then compared with the specified conditions to calculate MAE/Accuracy for numerical/categorical property.

\begin{table}[t]
\caption{Results on the Polymers dataset including four numerical conditions. The best-performing methods for each metric are highlighted in bold, and the second-best are underlined.}
\label{Polymers result}
\centering
\begin{adjustbox}{max width=\textwidth}
\begin{tabular}{ccccccccc}
\toprule
                        & \cellcolor[HTML]{FFFFFF}                           & \multicolumn{3}{c}{\cellcolor[HTML]{FFFFFF}Distribution Learning}                                                                       & \multicolumn{4}{c}{\cellcolor[HTML]{FFFFFF}Controllability}                                                                        \\
\multirow{-2}{*}{Model} & \multirow{-2}{*}{\cellcolor[HTML]{FFFFFF}Validity $\uparrow$} & \cellcolor[HTML]
{FFFFFF}Diversity $\uparrow$ & \cellcolor[HTML]{FFFFFF}Similarity $\uparrow$ & \cellcolor[HTML]{FFFFFF}FCD $\downarrow$ & \cellcolor[HTML]{FFFFFF}Synth. $\downarrow$& \cellcolor[HTML]{FFFFFF}$\text{O}_2$P. $\downarrow$& \cellcolor[HTML]{FFFFFF}$\text{N}_2$P. $\downarrow$& \cellcolor[HTML]{FFFFFF}$\text{CO}_2$P.$\downarrow$ \\ \hline
DiGress                                                                  & \underline{99.13}                                & \underline{0.9099}                                 & 0.2724                                  & 22.7237                           & 2.9842                              & 1.7163                              & 2.0630                              & 1.6738                               \\
DiGress v2                                                               & 98.12                                & \textbf{0.9105}                                 & 0.2771                                  & 21.7311                           & 2.7507                              & 1.7130                              & 2.0632                              & 1.6648                               \\
GDSS                                                                     & 92.05                                & 0.7510                                 & 0.0000                                  & 34.2627                           & 1.3701                              & 1.0271                             & 1.0820                              & 1.0683                               \\
MOOD                                                                    & 98.66                                & 0.8349                                 & 0.0227                                  & 39.3981                           & 1.4019                              & 1.4961                              & 1.7603                              & 1.4748                               \\
GraphDiT                                                                & 82.45                                & 0.8712                                 & \textbf{0.9600}                                  & \underline{6.6443}                           & 1.2973                              & 0.7440                              & 0.8857                              & 0.7550                               \\ \hline
\ours Absorb                                                             & 95.07                                & 0.8793                                 & \underline{0.9568}                                  & 6.8110                           & \underline{1.2572}                               & \underline{0.7311}                              & \underline{0.8754}                              & \underline{0.7466}                               \\
\ours Uniform                                                           & \textbf{99.38}                                & 0.8683                                 & 0.9541                                  & \textbf{6.4351}                           & \textbf{1.1905}                              & \textbf{0.6054}                              & \textbf{0.6993}                              & \textbf{0.6232}                               \\ \bottomrule
\end{tabular}
\end{adjustbox}
\end{table}

\begin{table}[h]
\caption{Comparisons with both numerical and categorical conditions evaluated by MAE and Accuracy, respectively. The best-performing methods for each metric are highlighted in bold, and the second-best are underlined.}
\label{HIV result}
\centering
\begin{adjustbox}{max width=\textwidth}
\begin{tabular}{cccccccc}
\toprule
                          &                         & \cellcolor[HTML]{FFFFFF}                           & \multicolumn{3}{c}{\cellcolor[HTML]{FFFFFF}Distribution Learning}                                    & \multicolumn{2}{c}{\cellcolor[HTML]{FFFFFF}Controllability}                \\
\multirow{-2}{*}{Dataset} & \multirow{-2}{*}{Model} & \multirow{-2}{*}{\cellcolor[HTML]{FFFFFF}Validity$\uparrow$} & \cellcolor[HTML]{FFFFFF}Diversity$\uparrow$ & \cellcolor[HTML]{FFFFFF}Similarity$\uparrow$ & \cellcolor[HTML]{FFFFFF}FCD$\downarrow$ & \cellcolor[HTML]{FFFFFF}MAE$\downarrow$ & \cellcolor[HTML]{FFFFFF}Accuracy$\uparrow$ \\ \hline
                          & DiGress                 & 43.77                                                  & \underline{0.9194}                                 & 0.8562                                  &13.0409                           & 1.9216                                  & 0.5335                                     \\
                          & DiGress v2              &50.50                                                  & 0.9193                                 & 0.8476                                  & 13.3997                           & 1.5934                                  & 0.5331                                     \\
                          & GDSS                    & 69.26                                                  & 0.7817                                 & 0.1032                                 & 45.3416                           & 1.2515                                  & 0.4830                                     \\
                          & MOOD                    & 28.75                                                  & \textbf{0.9280}                                 & 0.1361                                  & 32.3523                           & 2.3144                                  & 0.5106                                     \\
                          & GraphDiT                & 69.06                                                  & 0.8983                                 & \underline{0.9580}                                  &  6.0605                           &  0.3825                                  & 0.9437                                     \\ \cline{2-8} 
                          & SDGD Absorb             & \underline{76.30}                                                  & 0.8973                                 & \textbf{0.9584}                                  & \textbf{5.9495}                           & \textbf{0.2531}                                  & \underline{0.9755}                                     \\
\multirow{-7}{*}{\rotatebox[origin=c]{90}{HIV}}      & SDGD Uniform            & \textbf{83.36}                                                  & 0.8945                                 & 0.9519                                  & \underline{6.0499}                           & \underline{0.3560}                                  & \textbf{0.9898}                          \\ \hline
                          & DiGress                 & 69.60                                                  & 0.9098                                 & 0.6805                                  & 18.6921                           & 2.3658                                  & 0.6536                                     \\
                          & DiGress v2              & 68.92                                                  & \underline{0.9107}                                 & 0.6336                                  & 19.4498                           & 2.2694                                  & 0.6531                                     \\
                          & GDSS                    & 62.18                                                 & 0.8415                                 & 0.2672                                  & 39.9440                           & 1.3788                                  & 0.5037                                     \\
                          & MOOD                    & 80.08                                                  & \textbf{0.9273}                                 & 0.1715                                  & 34.2506                           & 2.0284                                  & 0.4903                                     \\
                          & GraphDiT                & 78.72                                                  & 0.8891                                 & 0.9281                                  & 11.3367                           & \underline{0.4492}                                  & 0.9089                                     \\ \cline{2-8} 
                          & SDGD Absorb             & \underline{86.24}                                                  & 0.8895                                 & \underline{0.9335}                                  & \underline{10.9960}                           & \textbf{0.4354}                                  & \underline{0.9398}                                     \\
\multirow{-7}{*}{\rotatebox[origin=c]{90}{BBBP}}     & SDGD Uniform            & \textbf{95.63}                                                  & 0.8807                                 & \textbf{0.9467}                                  & \textbf{10.8411}                           & 1.0745                                  & \textbf{0.9831}                                     \\ \hline

                    & DiGress                 & 35.11                                                  & \underline{0.8862}                                 & 0.6942                                  & 24.6560                           & 2.0681                                  & 0.5061                                     \\
                          & DiGress v2              & 35.46                                                  & 0.8812                                 & 0.7027                                  & 25.3270                           & 2.3365                                  & 0.5113                                     \\
                          & GDSS                    & 28.79                                                 & 0.8756                                 & 0.2708                                  & 46.7539                           & 1.6422                                  & 0.5036                                     \\
                          & MOOD                    & \textbf{99.47}                                                  & \textbf{0.8902}                                 & 0.2587                                  & 44.2394                           & 1.8853                                  & 0.5062                                     \\
                          & GraphDiT                & 76.00                                                  & 0.8256                                 & 0.8780                                  & 6.9663                           & \underline{0.4171}                                  & \underline{0.9132}                                     \\ \cline{2-8} 
                          & SDGD Absorb             & 77.20                                                  & 0.8251                                 & \underline{0.8786}                                  & \textbf{6.1112}                           & \textbf{0.3963}                                  & 0.8884                                     \\
\multirow{-7}{*}{\rotatebox[origin=c]{90}{BACE}}     & SDGD Uniform            & \underline{96.92}                                                  & 0.8133                                 & \textbf{0.8935}                                  & \underline{6.7618}                           & 0.8221                                  & \textbf{0.9250}   

\\ \bottomrule
\end{tabular}
\end{adjustbox}
\end{table}

\subsection{(RQ1) Overall Performance Comparison}
We first compare the overall performance of \ours with both CoG and PC against existing baselines across all the adopted datasets.

We present the results on the Polymers dataset in Table~\ref{Polymers result}, where \ours consistently achieves notable improvements over the most competing baseline GraphDiT in both validity and controllability. With a transition matrix leading to the \textit{Uniform} $p_{\text{base}}$, \ours attains an average relative improvement in MAE of +15.3\% compared to the strongest baseline.
While DiGress and its v2 variant generate diverse molecules, they struggle to satisfy multiple conditions simultaneously, leading to weaker controllability and lower FCD scores. In contrast, \ours not only maintains strong distribution learning but also achieves the best FCD, confirming the effectiveness of \ours.

Results on datasets with both numerical and categorical conditions (see Table~\ref{HIV result}) further confirm these findings. \ours achieves the best validity and controllability, while maintaining competitive distribution learning, particularly in terms of FCD and similarity. These results demonstrate that \ours can flexibly guide the generation of molecular graphs in diverse conditional settings, producing valid molecules with improved compliance to the specified properties.

\subsection{Ablation Studies}
To assess the contribution of each component, we conduct a detailed ablation study focusing on two aspects: score-based modeling in graph diffusion and two score manipulation techniques.

\subsubsection{(RQ2) Performance of Score-based Modeling}
Score-based modeling is evaluated against the current mean-prediction approach, GraphDiT, using the same guidance strategy. 
Due to space constraints, we report results on two datasets in Table~\ref{Ablation1}, with the complete results across all four datasets provided in Appendix (see Table~\ref{full_Ablation1}).
For clarity, results on the Polymers dataset are reported as the average MAE of gas permeability across all gases, rather than per-gas MAE. 

The most notable improvement is in molecular validity, with \ours achieving better performance under both \textit{Absorb} and \textit{Uniform} $p_{\text{base}}$. Improvements in controllability also remain competitive. Taken together, these findings indicate that score-based modeling substantially improves the validity and controllability of discrete graph diffusion in multi-conditional generation.

\begin{table}[t]
\centering
\caption{Ablation (RQ2): Comparing \ours (w/o PC, and with CFG rather than CoG) to GraphDiT with CFG, where numerical and categorical conditions are evaluated using MAE and Accuracy, respectively.}
\label{Ablation1}
\begin{adjustbox}{max width=\textwidth}
\begin{tabular}{cccccc}
\toprule
\multirow{2}{*}{Dataset}                      & \multirow{2}{*}{Model} & \multirow{2}{*}{Validity$\uparrow$} & \multicolumn{3}{c}{Controllability} \\ \cline{4-6} 
                                              &                        &                           & Synth.$\downarrow$      & $\text{Gas Perm.}\downarrow$     & $\text{Inhibition}\uparrow$    \\ \hline
\multicolumn{1}{l}{\multirow{3}{*}{Polymers}} & GraphDiT               & 82.45                     & 1.2973          & 0.9205        & -      \\
\multicolumn{1}{l}{}                          & \ours Absorb            & \underline{90.15}                     & \textbf{1.1129}          & \textbf{0.8940}   & -          \\
\multicolumn{1}{l}{}                          & \ours Uniform           & \textbf{95.81}                     & \underline{1.2779}          & \underline{0.9179} & -           \\ \hline
\multirow{3}{*}{HIV}                          & GraphDiT               & 69.06                     & 0.3825      & -    & 0.9437            \\
                                              & \ours Absorb            & \underline{74.01}                     & \textbf{0.3035}       & -   & \underline{0.9761}            \\
                                              & \ours Uniform           & \textbf{81.30}                     & \underline{0.3650}     & -     & \textbf{0.9898}            \\  \bottomrule
\end{tabular}
\end{adjustbox}
\end{table}

\begin{table}[th]
\centering
\caption{Ablation (RQ3): Comparison of CoG and CFG in \ours with and without PC. Numerical and categorical conditions are evaluated using MAE and Accuracy, respectively.}
\label{Ablation2}
\begin{adjustbox}{max width=\textwidth}
\begin{tabular}{cccccccc}
\toprule
\multirow{2}{*}{Dataset}      & \multirow{2}{*}{Strategy} & \multirow{2}{*}{\begin{tabular}[c]{@{}c@{}}Probability\\ Calibration\end{tabular}} & \multirow{2}{*}{Validity$\uparrow$} & \multicolumn{3}{c}{Controllability} \\ \cline{5-7} 
                                                         &                           &                                                                                    &                           & Synth.$\downarrow$       & $\text{Gas Perm.}\downarrow$ & $\text{Inhibition}\uparrow$   \\ \toprule
\multirow{4}{*}{Polymers} 
                           & \multirow{2}{*}{CFG}      & $\times$                                                                           & 95.81                     & 1.2779           & 0.9179            & - \\
                           &                           & $\surd$                                                                            & 99.51            & 1.1815  & 0.8088  & -  \\ 
                           & \multirow{2}{*}{CoG}      & $\times$                                                                           & 98.84                     & 1.1784           & 0.7968       & -      \\
                           &                           & $\surd$                                                                            & 99.36            & 1.1905           & 0.7796   & - \\ \hline
\multirow{4}{*}{HIV}      & \multirow{2}{*}{CFG}      & $\times$                                                                           & 81.30                       & 0.3650  & -         & 0.9898           \\
                          &                           & $\surd$                                                                            & 82.34            & 0.3418  & - & 0.9910  \\ 
                          & \multirow{2}{*}{CoG}      & $\times$                                                                           & 80.89                     & 0.3855          & -   & 0.9869           \\
                          &                           & $\surd$                                                                            & 83.36            & 0.3560  & -  & 0.9898  \\  \bottomrule
\end{tabular}
\end{adjustbox}
\end{table}

\subsubsection{(RQ3) Effectiveness of Composable Guidance and Probability Calibration}

As shown in Table~\ref{Ablation2}, we evaluate the effectiveness of CoG on two datasets, with results on all four benchmarks provided in the Appendix (see Table~\ref{full_Ablation2}). Under the same network architecture, hyperparameters, and guidance scale, CoG achieves performance comparable to CFG with a \textit{Uniform} $p_{\text{base}}$. For instance, \ours with CoG attains 98.84\% validity while maintaining strong controllability. Across other datasets, CoG consistently matches CFG in performance while enabling flexible sampling of \ours, underscoring its practicality for multi-conditional molecular generation.

We further assess the effect of PC under both CFG and CoG. Across all datasets and guidance strategies, incorporating PC consistently enhances validity and controllability. Despite its simplicity, PC proves to be a practical and effective tool, serving as a general, plug-and-play component to improve diverse score-manipulation techniques in molecular graph generation.

\begin{table}[th]
\centering
\caption{(RQ4): Comparison of \ours performance when sampling with different subsets of properties specified. $^*$ Total guidance weight of each subset is kept the same. Synth.: Synthesizability; $\text{O}_2$, $\text{N}_2$,  $\text{CO}_2$: Gas Permeability.}
\label{Ablation3}
\begin{tabular}{ccccccc}
\toprule
\multirow{2}{*}{Model}        & \multirow{2}{*}{\begin{tabular}[c]{@{}c@{}}Property\\ $\text{Subsets}^*$\end{tabular}} & \multirow{2}{*}{Validity$\uparrow$} & \multicolumn{4}{c}{Controllability} \\ \cline{4-7} 
                              &                                                                             &                           & Synth.$\downarrow$   & $\text{O}_2$P.$\downarrow$   & $\text{N}_2$P.$\downarrow$   & $\text{CO}_2$P.$\downarrow$  \\ \hline
\multirow{1}{*}{GraphDiT}  & {\{}Synth., $\text{O}_2$, $\text{N}_2$,  $\text{CO}_2${\}}                                            & 82.45                         &1.2973        & 0.7440      & 0.8857      & 0.7550      \\
                             \hline
\multirow{4}{*}{\ours} & {\{}Synth., $\text{O}_2$, $\text{N}_2$,  $\text{CO}_2${\}}                                           & 99.38                         & 1.1905        & 0.6054      & 0.6993      & 0.6232      \\
                              & {\{}Synth.{\}}                                                                & 99.78                         & 0.9179        & -      & -      & -      \\
                              & {\{}$\text{O}_2$, $\text{N}_2$,  $\text{CO}_2${\}}                                                                 & 99.38                         & -        & 0.6001      & 0.6927      & 0.6133      \\
                              & {\{}Synth., $\text{O}_2${\}}                                                        & 99.50                         & 1.1894        & 0.6099      & -      & -      \\ \bottomrule
\end{tabular}
\end{table}

\subsection{(RQ4) Flexibility}
To evaluate the flexibility of \ours, we conduct experiments on the Polymers dataset, which contains two types of four properties and is well-suited for this analysis. Using the same trained models, we assess \ours under different property subsets, including the full set, single-type subsets (e.g., Synthesizability alone or only Gas Permeability properties), and mixed-type subsets (e.g., Synthesizability combined with a single Gas Permeability property).

As shown in Table~\ref{Ablation3}, CoG enables flexible sampling by conditioning generation on arbitrary subsets of properties. When conditioning solely on Synthesizability or the Gas Permeability properties ($\text{O}_2$, $\text{N}_2$, $\text{CO}_2$), the model achieves the best controllability for the targeted property while maintaining high validity, suggesting that focusing on a single property allows more precise control. For mixed subsets (e.g., Synthesizability combined with $\text{O}_2$), the model effectively balances multiple objectives, achieving strong validity and good controllability on the selected properties. Overall, these results demonstrate that CoG supports flexible and practical multi-conditional generation, enabling molecular design to be tailored to diverse application needs.
\subsection{Efficiency}
\label{subsec:limitation}
In the above experiments, we evaluate the performance and flexibility of \ours using CoG. However, a natural concern arises: the inference time of CoG grows linearly with the number of specified properties, which may limit its practicality in efficiency-sensitive scenarios. To address this, in Sec.~\ref{subsec:composable}, we introduce a training strategy that enables \ours to support both CoG and a more efficient variant, fast-CoG, simultaneously.

As shown in Table~\ref{TimeCost}, we measure the average inference time per molecule on the test set using \ours trained with that training strategy, across different property subsets and guidance strategies. The results confirm that the inference time of CoG is longer. In contrast, fast-CoG achieves comparable performance while maintaining more efficient sampling. This demonstrates the strong flexibility of \ours: a single trained model can switch between CoG and fast-CoG with different property subsets at inference, allowing users to balance controllability and efficiency according to the needs of different scenarios.

\begin{table}[th]
\centering
\caption{Evaluation of inference efficiency of \ours using CoG and fast-CoG across different subsets of properties specified. All results are derived from the same trained model, using different sampling settings at inference. $^*$ denotes the inference time per molecule, measured in seconds.}
\label{TimeCost}
\begin{adjustbox}{max width=\textwidth}
\begin{tabular}{cccccccc}
\toprule
\multirow{2}{*}{\begin{tabular}[c]{@{}c@{}}Property\\ Subsets\end{tabular}} & \multirow{2}{*}{Strategy} & \multirow{2}{*}{\begin{tabular}[c]{@{}c@{}}Time$^*$\\ Cost\end{tabular}} &  \multirow{2}{*}{Validity$\uparrow$} & \multicolumn{4}{c}{Controllability}                                                                                                                                                                           \\ \cline{5-8} 
                                                                            &                                                                              &                                                                      &                                       & Synth.$\downarrow$   & $\text{O}_2$P.$\downarrow$   & $\text{N}_2$P.$\downarrow$   & $\text{CO}_2$P.$\downarrow$ \\ \hline
\multirow{2}{*}{{\{}Synth., $\text{O}_2$, $\text{N}_2$,  $\text{CO}_2${\}}  }  & CoG     & 177    & 99.38     & 1.1905       & 0.6054      & 0.6993      & 0.6232   \\
                                                                              & fast-CoG    & \textbf{72}    & 99.48    & 1.1874    & 0.6079      & 0.6998    & 0.6292     \\ \hline
\multirow{2}{*}{{\{}Synth.{\}} }     & CoG      & 73     & 99.78         & 0.9179        & -      & -      & - \\
                                     & fast-CoG      & \textbf{71}    & 99.78      & 0.9179       & -     & -       & -          \\ \hline
\multirow{2}{*}{{\{}$\text{O}_2$, $\text{N}_2$,  $\text{CO}_2${\}} }           & CoG           & 144         & 99.38                         & -        & 0.6001      & 0.6927      & 0.6133         \\
                                                                            & fast-CoG                 & \textbf{73}               & 99.41       & -          & 0.6003      & 0.6950         & 0.6141         \\ \hline
\multirow{2}{*}{ {\{}Synth., $\text{O}_2${\}} }         & CoG         & 109         & 99.50           & 1.1894        & 0.6099      & -      & -       \\
                                                         & fast-CoG    & \textbf{73}          & 99.44       & 1.1732        & 0.6297        & -         & -         \\ \bottomrule
\end{tabular}
\end{adjustbox}
\end{table}

\section{Conclusion and Future Directions}
In this work, inspired by the flexibility of score-based modeling in continuous-state diffusion, we introduce concrete scores for discrete graph diffusion. By bridging the gap between discrete molecular structures and score-based modeling, concrete scores open the door to applying a wide range of score-based techniques to discrete graph domains. To demonstrate their flexibility and effectiveness, we propose two novel strategies and validate them on the challenging task of multi-conditional molecular generation. We demonstrate that \ours achieves strong performance and exhibits high flexibility on several multi-conditional molecular generation benchmarks. 

Recent progress in discrete flow matching—notably Discrete Flow Matching (DFM) for general discrete data~\cite{gat2024discrete} and its graph-specialized variant DeFoG~\cite{qinmadeira2024defog}—offers an alternative generative paradigm that disentangles training and sampling for increased efficiency. In contrast, our method operates within a score-based diffusion framework, enabling principled score composability and flexible multi-conditional control; exploring hybrid approaches that combine composable score guidance with the sampling efficiency of discrete flow matching is an exciting direction for future work.
Moreover, we hope to extend our framework to a wider range of graph generation domains such as citation and social networks, providing a unified foundation for controllable graph generative modeling.

\bibliography{sample}
\newpage
\clearpage
\section{Full Ablation Results}
Tables~\ref{full_Ablation1} and~\ref{full_Ablation2} provide the complete ablation results across all four datasets, complementing the condensed versions shown in the main text. Table~\ref{full_Ablation1} compares \ours (without PC) with CFG against GraphDiT using the same guidance strategy, reporting validity and controllability metrics (Avg. MAE for Polymers, Property Accuracy for HIV, BBBP, and BACE). Table~\ref{full_Ablation2} further examines the effects of Composable Guidance (CoG) and Probability Calibration (PC) under both CFG and CoG strategies, highlighting all metrics improved by PC. These full tables allow a comprehensive view of our method’s performance and ablation effects across all datasets.

\begin{table}[h]
\centering
\caption{Ablation (RQ2): Comparing \ours (w/o PC, and with CFG rather than CoG) to GraphDiT with CFG, where numerical and categorical conditions are evaluated using MAE and Accuracy, respectively.}
\label{full_Ablation1}
\begin{adjustbox}{max width=\textwidth}
\begin{tabular}{cccccc}
\toprule
\multirow{2}{*}{Dataset}                      & \multirow{2}{*}{Model} & \multirow{2}{*}{Validity$\uparrow$} & \multicolumn{3}{c}{Controllability} \\ \cline{4-6} 
                                              &                        &                           & Synth.$\downarrow$      & $\text{Gas Perm.}\downarrow$   & $\text{Accuracy}\uparrow$  \\ \hline
\multicolumn{1}{l}{\multirow{3}{*}{Polymers}} & GraphDiT               & 82.45                     & 1.2973          & 0.9205  & -          \\
\multicolumn{1}{l}{}                          & \ours Absorb            & \underline{90.15}                     & \textbf{1.1129}          & \textbf{0.8940}    & -        \\
\multicolumn{1}{l}{}                          & \ours Uniform           & \textbf{95.81}                     & \underline{1.2779}          & \underline{0.9179}       & -     \\ \hline
\multirow{3}{*}{HIV}                          & GraphDiT               & 69.06                     & 0.3825      & -    & 0.9437            \\
                                              & \ours Absorb            & \underline{74.01}                     & \textbf{0.3035}         & - & \underline{0.9761}            \\
                                              & \ours Uniform           & \textbf{81.30}                     & \underline{0.3650}        & -  & \textbf{0.9898}            \\ \hline
\multirow{3}{*}{BBBP}                         & GraphDiT               & 78.72                     & \textbf{0.4492}        & -  & 0.9089            \\
                                              & \ours Absorb            & \underline{86.58}                     & \underline{0.4625}        & -  & \underline{0.9449}            \\
                                              & \ours Uniform           & \textbf{88.02}                     & 1.1669       & -   & \textbf{0.9725}            \\ \hline
\multirow{3}{*}{BACE}                         & GraphDiT               & 76.00                     & \underline{0.4171}      & -    & \underline{0.9132}            \\
                                              & \ours Absorb            & \underline{78.08}                     & \textbf{0.3904}       & -   & 0.8998            \\
                                              & \ours Uniform           & \textbf{92.23}                     & 0.8683      & -    & \textbf{0.9290}            \\ \bottomrule
\end{tabular}
\end{adjustbox}
\end{table}

\begin{table}[h]
\centering
\caption{Ablation (RQ3): Comparison of CoG and CFG in \ours with and without PC. Metrics improved by PC are highlighted in bold. Numerical and categorical conditions are evaluated using MAE and Accuracy, respectively.}
\label{full_Ablation2}
\begin{adjustbox}{max width=\textwidth}
\begin{tabular}{cccccccc}
\hline
\multirow{2}{*}{Dataset}  & \multirow{2}{*}{\begin{tabular}[c]{@{}c@{}}Transition\\ Matrix\end{tabular}}        & \multirow{2}{*}{Strategy} & \multirow{2}{*}{\begin{tabular}[c]{@{}c@{}}Probability\\ Calibration\end{tabular}} & \multirow{2}{*}{Validity$\uparrow$} & \multicolumn{3}{c}{Controllability} \\ \cline{6-8} 
                          &                               &                           &                                                                                    &                           & Synth.$\downarrow$        & $\text{Gas Perm.}\downarrow$ & $\text{Accuracy}\uparrow$   \\ \hline
\multirow{8}{*}{Polymers} & \multirow{4}{*}{Absorb}  & \multirow{2}{*}{CFG}      & $\times$                                                                           & 90.15                     & 1.1129           & 0.8940         & -  \\
                          &                               &                           & $\surd$                                                                            & \textbf{90.89}            & 1.1264           & \textbf{0.8791} & - \\ \cline{3-8} 
                          &                               & \multirow{2}{*}{CoG}      & $\times$                                                                           & 94.89                     & 1.2506           & 0.9066         & -  \\
                          &                               &                           & $\surd$                                                                            & \textbf{95.07}            & 1.2572           & \textbf{0.9026} & - \\ \cline{2-8} 
                          & \multirow{4}{*}{Uniform} & \multirow{2}{*}{CFG}      & $\times$                                                                           & 95.81                     & 1.2779           & 0.9179        & -   \\
                          &                               &                           & $\surd$                                                                            & \textbf{99.51}            & \textbf{1.1815}  & \textbf{0.8088} & - \\ \cline{3-8} 
                          &                               & \multirow{2}{*}{CoG}      & $\times$                                                                           & 98.84                     & 1.1784           & 0.7968        & -   \\
                          &                               &                           & $\surd$                                                                            & \textbf{99.36}            & 1.1905           & \textbf{0.7796} & - \\ \hline
\multirow{8}{*}{HIV}      & \multirow{4}{*}{Absorb}  & \multirow{2}{*}{CFG}      & $\times$                                                                           & 74.01                     & 0.3035         & -  & 0.9761           \\
                          &                               &                           & $\surd$                                                                            & \textbf{75.95}            & \textbf{0.2840}  & - & \textbf{0.9770}  \\ \cline{3-8} 
                          &                               & \multirow{2}{*}{CoG}      & $\times$                                                                           & 75.74                     & 0.2548          & -  & 0.9745           \\
                          &                               &                           & $\surd$                                                                            & \textbf{76.30}            & \textbf{0.2531} & -  & \textbf{0.9755}  \\ \cline{2-8} 
                          & \multirow{4}{*}{Uniform} & \multirow{2}{*}{CFG}      & $\times$                                                                           & 81.30                     & 0.3650      & -     & 0.9898           \\
                          &                               &                           & $\surd$                                                                            & \textbf{82.34}            & \textbf{0.3418} & - & \textbf{0.9910}  \\ \cline{3-8} 
                          &                               & \multirow{2}{*}{CoG}      & $\times$                                                                           & 80.89                     & 0.3855         & -  & 0.9869           \\
                          &                               &                           & $\surd$                                                                            & \textbf{83.36}            & \textbf{0.3560} & - & \textbf{0.9898}  \\ \hline
\multirow{8}{*}{BBBP}     & \multirow{4}{*}{Absorb}  & \multirow{2}{*}{CFG}      & $\times$                                                                           & 86.58                     & 0.4625         & -  & 0.9449           \\
                          &                               &                           & $\surd$                                                                            & 86.15                     & \textbf{0.4494} & - & \textbf{0.9458}  \\ \cline{3-8} 
                          &                               & \multirow{2}{*}{CoG}      & $\times$                                                                           & 85.12                     & 0.4375         & -  & 0.9449           \\
                          &                               &                           & $\surd$                                                                            & \textbf{86.24}            & \textbf{0.4354} & - & 0.9398           \\ \cline{2-8} 
                          & \multirow{4}{*}{Uniform} & \multirow{2}{*}{CFG}      & $\times$                                                                           & 88.02                     & 1.1669         & -  & 0.9725           \\
                          &                               &                           & $\surd$                                                                            & \textbf{94.42}            & \textbf{1.0707} & - & \textbf{0.9793}  \\ \cline{3-8} 
                          &                               & \multirow{2}{*}{CoG}      & $\times$                                                                           & 86.09                     & 1.1801         & -  & 0.9705           \\
                          &                               &                           & $\surd$                                                                            & \textbf{95.63}            & \textbf{1.0745} & - & \textbf{0.9831}  \\ \hline
\multirow{8}{*}{BACE}     & \multirow{4}{*}{Absorb}  & \multirow{2}{*}{CFG}      & $\times$                                                                           & 78.08                     & 0.3904      & -     & 0.8998           \\
                          &                               &                           & $\surd$                                                                            & \textbf{78.18}            & \textbf{0.3891} & - & \textbf{0.9073}  \\ \cline{3-8} 
                          &                               & \multirow{2}{*}{CoG}      & $\times$                                                                           & 77.20                     & 0.4033
                          & - & 0.8874           \\
                          &                               &                           & $\surd$                                                                            & \textbf{77.42}            & \textbf{0.3964}  & - & \textbf{0.8884}  \\ \cline{2-8} 
                          & \multirow{4}{*}{Uniform} & \multirow{2}{*}{CFG}      & $\times$                                                                           & 92.23                     & 0.8683      & -     & 0.9290           \\
                          &                               &                           & $\surd$                                                                            & \textbf{95.74}            & \textbf{0.8326} & -  & \textbf{0.9395}  \\ \cline{3-8} 
                          &                               & \multirow{2}{*}{CoG}      & $\times$                                                                           & 95.81                     & 0.8547         & -   & 0.9183           \\
                          &                               &                           & $\surd$                                                                            & \textbf{96.92}            & \textbf{0.8221}  & - & \textbf{0.9250}  \\ \hline
\end{tabular}
\end{adjustbox}
\end{table}

\section{Supplementary Details about Experiments}
All data and code used in this work are publicly available at our GitHub repository: https://github.com/anjie-qiao/CSGD.

All experiments are trained and evaluated on a single NVIDIA A800-80GB-PCIE GPU.
We provide default configuration files and fixed random seeds to ensure that all reported results can be fully reproduced. Training \ours on each of these four datasets can be completed within 24 hours. The final checkpoint is used for evaluation.

\subsubsection{Datasets}
We evaluate \ours on four datasets: one for materials design (Polymers) and three for drug design (HIV, BACE, and BBBP). he Polymers dataset contains four numerical conditions. The first condition, synthesizability, combines two scores—Synthetic Accessibility Score (SAS) and Synthetic Complexity Score (SCS)—concatenated as a two-dimensional vector, e.g., [6.1, 2.29]. The remaining three conditions correspond to gas permeability ($\text{O}_2$, $\text{N}_2$, and $\text{CO}_2$), each represented by a single numerical value, e.g., 2.8, 1.5, and 15.4.

For the drug design datasets, each sample includes both a numerical and a categorical condition. The numerical condition is the same two-dimensional synthesizability vector used in Polymers. The categorical condition is a binary label indicating a task-specific property, e.g., for HIV, whether a molecule inhibits HIV replication ($1$ indicates inhibition, $0$ otherwise). All three drug design datasets are class-balanced.

Following GraphDiT, We randomly split each dataset into training, validation, and test sets with a 6:2:2 ratio. A fixed random seed is used to ensure that the splits are reproducible and do not affect result replication.

\subsubsection{Diffusion Details}
Our score-based diffusion model is built on the official SEDD implementation. Following SEDD, we use a log-linear noise schedule: $\bar{\sigma}(t) = -\log\bigl(1 - (1-\epsilon t))$, where $\epsilon = 10^{-5}$ is a small constant to ensure numerical stability as $t \to 1$.

We adopt the architecture and hyperparameters from GraphDiT, with only minor changes such as adjusting epochs and learning rate. Specifically, all models are trained with a batch size of 1200, a learning rate of $3\times 10^{-4}$, and gradient clipping with a maximum norm of 1.
A linear warm-up is applied for the first 1500 iterations.
The number of diffusion steps is set to 1000.

\subsubsection{Condition Encoders} 
\label{appendix:encoder}
To enable multi-conditional molecular generation, we employ separate encoders for categorical and numerical properties, following GraphDiT.

\medskip
\noindent\textbf{Categorical Encoder.} 
Categorical conditions (e.g., task-specific labels such as HIV inhibition) are embedded using a \textit{Categorical Encoder}. Each category $c \in \{0, \dots, K-1\}$ is mapped to a learnable vector $\mathbf{e}_c \in \mathbb{R}^{d_h}$:
\begin{align}
\mathbf{e}_c &= \mathrm{Embed}(c'), &
c' &=
\begin{cases} 
K, & \text{with probability } p_\mathrm{drop}, \\
c, & \text{otherwise},
\end{cases}
\end{align}
where $\mathrm{Embed}(\cdot)$ is a learnable table of size $(K+1) \times d_h$. Dropout of individual labels allows simultaneous learning of conditional and unconditional estimates.

\noindent\textbf{Cluster Encoder.} 
Numerical conditions $c \in \mathbb{R}^{d_\mathrm{in}}$ are embedded using a \textit{Cluster Encoder}. The embedding is given by
\begin{align}
\mathbf{e}_c &=
\begin{cases} 
\mathbf{e}_\varnothing, & \text{with probability } p_\mathrm{drop}, \\
\mathrm{MLP}(c), & \text{otherwise},
\end{cases} 
\end{align}
where $\mathrm{MLP}(\cdot)$ is a two-layer network with an intermediate $\mathrm{Softmax}$, and $\mathbf{e}_\varnothing$ is a learnable null embedding shared across dropped conditions.

These two encoders together allow diffusion model to flexibly incorporate arbitrary combinations of categorical and numerical conditions during both training and sampling.

\begin{figure}[t]
    \centering
    \label{case}
    \includegraphics[width=0.95\textwidth]{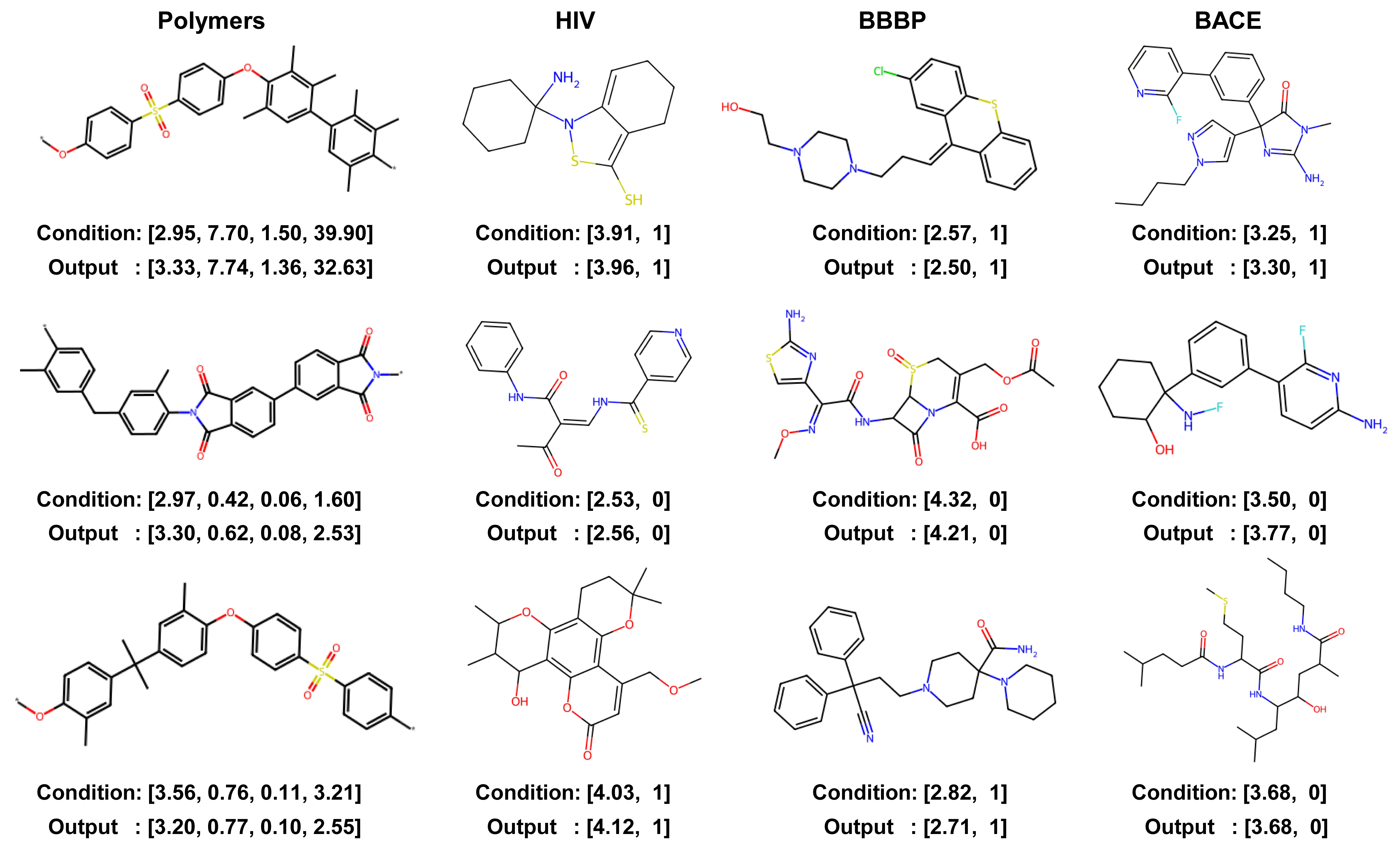}
    \caption{Generated molecular graphs across the four datasets. 
The \textbf{Condition} refer to the input properties, while the \textbf{Output} correspond to the properties of the generated molecules. 
For the Polymers dataset, the four properties are Synthesizability, and Gas Permeabilities of O$_2$, N$_2$, and CO$_2$. 
For the HIV, BBBP, and BACE datasets, the two properties are Synthesizability and a task-specific property (e.g., HIV inhibition).}
\end{figure}

\subsubsection{The Denoising Network.}
Our denoising network is built upon the GraphDiT architecture. Given a noisy molecular graph at timestep $t$, the node and edge features are first concatenated and projected into a hidden space $\textbf{H}$ via a linear layer. These representations are then processed by a stack of Graph Transformer~\cite{liu2024_graphdit} (GT) layers equipped with adaptive layer normalization~\cite{huang2017_adaln} (AdaLN), where the normalization parameters are modulated by the combined condition and timestep embeddings $\text{e}_y = \text{e}_c+\text{e}_t$. Formally, the hidden states $\mathbf{H}$ are updated as
$$\mathbf{H} = \mathrm{GT}(\mathbf{H}, \text{e}_y),$$ 
integrating conditional information. The final hidden states are passed through an MLP with AdaLN: $$\mathbf{H}_{\text{output}} = \mathrm{AdaLN}(\mathrm{MLP}(\mathbf{H}), \text{e}_y),$$ 
and then split into atom and bond components, $\mathbf{H}^\text{X}_{\text{output}}$ and $\mathbf{H}^\text{E}_{\text{output}}$, which correspond to the predicted node and edge log-concrete scores $s^X_{\theta}$ and $s^E_{\theta}$. This design allows the network to jointly reason about atom types and bond structures while flexibly incorporating both categorical and numerical conditions.

\subsection{Visualization}
We visualize samples generated by \ours across the four datasets, as shown in Figure~\ref{case}. 
These examples demonstrate that \ours is able to generate molecules that closely adhere to the the specified input multi-conditions.

\end{document}